\documentclass[final, 12pt]{elsarticle}
\usepackage{physics}
\usepackage[shortlabels]{enumitem}
\usepackage{mathtools}
\usepackage{amsthm}
\usepackage[utf8]{inputenc}
\usepackage{graphicx}
\graphicspath{ {images/} }
\usepackage[english]{babel}
\usepackage[margin=1in]{geometry}
\usepackage{url}
\usepackage{amsmath, nccmath}
\usepackage{float}
\usepackage{caption}
\usepackage{comment}
\usepackage{lscape}
\usepackage{afterpage}
\DeclareMathOperator{\orness}{orness}
\DeclareMathOperator{\andness}{andness}

\newcommand{\diff}{\,\text{d}}
\newcommand{\bslash}{\,\setminus\,}
\newcommand{\comp}{\text{co}}
\newcommand\restr[2]{{
		\left.\kern-\nulldelimiterspace 
		#1 
		\vphantom{\big|} 
		\right|_{#2} 
}}

\newcommand{\tnorm}{\mathcal{T}}
\newcommand{\imp}{\mathcal{I}}
\usepackage[utf8]{inputenc}
\usepackage{lmodern}
\usepackage{mathtools}
\newtheorem{thm}{Theorem}[section] 
\newtheorem{prop}[thm]{Proposition}

\newtheorem{cor}[thm]{Corollary}

\newtheorem{defn}[thm]{Definition} 
\newtheorem{exmp}[thm]{Example}
\usepackage{amssymb}

\usepackage{hyperref}
\journal{arxiv.org}

\begin{document}

\begin{frontmatter}

\title{Choquet-Based Fuzzy Rough Sets}

\author[mymainaddress]{Adnan Theerens\corref{mycorrespondingauthor}}
\cortext[mycorrespondingauthor]{Corresponding author}
\ead{adnan.theerens@ugent.be}
\author[mymainaddress]{Oliver Urs Lenz}
\ead{oliver.lenz@ugent.be}
\author[mymainaddress]{Chris Cornelis}
\ead{chris.cornelis@ugent.be}

\address[mymainaddress]{Computational Web Intelligence, Department of Applied Mathematics, Computer Science and Statistics, Ghent University, Ghent, Belgium}

\begin{abstract}
Fuzzy rough set theory can be used as a tool for dealing with inconsistent data when there is a gradual notion of indiscernibility between objects. It does this by providing lower and upper approximations of concepts.
In classical fuzzy rough sets, the lower and upper approximations are determined using the minimum and maximum operators, respectively. This is undesirable for machine learning applications, since it makes these approximations sensitive to outlying samples. To mitigate this problem, ordered weighted average (OWA) based fuzzy rough sets were introduced. In this paper, we show how the OWA-based approach can be interpreted intuitively in terms of vague quantification, and then generalize it to Choquet-based fuzzy rough sets (CFRS). This generalization maintains desirable theoretical properties, such as duality and monotonicity. Furthermore, it provides more flexibility for machine learning applications. In particular, we show that it enables the seamless integration of outlier detection algorithms, to enhance the robustness of machine learning algorithms based on fuzzy rough sets.
\end{abstract}

\end{frontmatter}

\section{Introduction}
Rough set theory, introduced by Pawlak \cite{pawlak1982rough}, provides a lower and upper approximation of a concept with respect to the indiscernibility relation between objects. The lower and upper approximation contain all objects that are certainly, resp.\ possibly part of the concept. That is to say, an element is a member of the lower approximation of a concept if every element indiscernible from it belongs to the concept; and an element is a member of the upper approximation of the concept if there exists an element indiscernible from it that belongs to the concept. Rough set theory was first extended to fuzzy rough set theory by Dubois and Prade \cite{dubois1990rough}, here both the concept and the indiscernibility relation can be fuzzy. Fuzzy rough set theory has been used successfully for classification and other machine learning purposes, such as feature and instance selection \cite{vluymans2015applications}, but due to the fact that the approximations in classical fuzzy rough sets are determined using the minimum and maximum operators, these approximations (and the applications based on them) are sensitive to noisy and outlying samples. To mitigate this problem, many noise-tolerant versions of fuzzy rough sets (FRS) have been proposed, such as Vaguely Quantified FRS \cite{cornelis2007vaguely}, \(\beta\)-Precision FRS \cite{FERNANDEZSALIDO2003547,FERNANDEZSALIDO2}, Variable Precision FRS \cite{mieszkowicz2004variable}, Variable Precision \((\theta,\sigma)\)-FRS \cite{YAO201458}, Soft Fuzzy Rough Sets \cite{HU20104384}, Automatic Noisy Sample Detection FRS \cite{HADRANI202037}, Data-Distribution-Aware FRS \cite{data-distribution-awareFR}, Probability Granular Distance based FRS \cite{PGDFRS} and Ordered Weighted Averaging (OWA) based FRS \cite{cornelis2010ordered}. \\
The Choquet integral, which is commonly used in decision making \cite{grabisch2010decade}, is a generalization of the Lebesgue integral to non-additive measures. It induces an interesting class of aggregation operators, that contains the weighted mean and OWA operators as special cases.\\
In this paper, we introduce a generalization of OWA-based fuzzy rough sets (OWAFRS), called Choquet-based fuzzy rough sets (CFRS), that uses the Choquet integral to determine the lower and upper approximation of a concept. This adds extra flexibility for machine learning purposes, while still retaining the important theoretical properties that OWAFRS has, such as monotonicity w.r.t.\ the indiscernibility relation and duality. We also show how OWAFRS can be interpreted in terms of vague quantification. Furthermore, we explain and demonstrate how to combine CFRS and normalized outlier scores \cite{normalizedLOF} to boost the robustness of the lower and upper approximations in fuzzy rough sets.
\\This paper is structured as follows: in Section \ref{prelims}, we recall the required prerequisites for (OWA-based) fuzzy rough sets and Choquet integration, while Section \ref{sec: vague quantification} discusses vague quantification. Section \ref{choq-basedFuzzyRoughSets_section} combines the previous sections to link OWAFRS with vague quantification and introduces CFRS together with several non-symmetric monotone measures that can be used with CFRS. These measures have a concrete interpretation in terms of vague quantification and let us smoothly combine outlier detection algorithms and fuzzy rough sets, moreover they cannot be realised using the OWAFRS approach. In Section \ref{section:application}, these measures are applied and evaluated for classification. Section \ref{section: conclusion and future work} concludes this paper and describes opportunities for future research.
\section{Preliminaries}
\label{prelims}
\subsection{Fuzzy set theory}
In this subsection, we recall the necessary notions of fuzzy set and fuzzy logical connectives. We start with the definition of a fuzzy set and a fuzzy relation.
\begin{defn}{\cite{fuzzysetsss}}
	A \emph{fuzzy set} or \emph{membership function} \(A\) on \(X\) is a function from \(X\) to the unit interval, i.e.\ \(A:X\to [0,1]\). The value \(A(x)\) of an element \(x\in X\) is called the \emph{degree of membership} of \(x\) in the fuzzy set \(A\). The set of all fuzzy sets on \(X\) is denoted as \(\mathcal{F}(X)\). A fuzzy relation \(R\) on \(X\) is an element of \(\mathcal{F}(X\times X)\).
\end{defn}
\begin{defn}{\cite{fuzzysetsss}}
The notation \(A\subseteq B\) for two fuzzy sets \(A\) and \(B\), expresses that \(A(x)\leq B(x)\) for all \(x\in X\). The fuzzy set \(A\cap B \in \mathcal{F}(X)\) is defined by \((A\cap B)(x)=\min(A(x),B(x))\).
\end{defn}
We will also make use of conjunctors, implicators and negators which extend their Boolean counterparts to the fuzzy setting.
\begin{defn}
	\hfill
\begin{itemize}
\item A function \(\mathcal{C}:[0,1]^2\to [0,1]\) is called a \emph{conjunctor} if it is increasing in both arguments and satisfies \(\mathcal{C}(0,0)=\mathcal{C}(1,0)=\mathcal{C}(0,1)=0\), \(\mathcal{C}(1,1)=1\) and \(\mathcal{C}(1,x)=x\) for all \(x\in[0,1]\). A commutative and associative conjunctor \(\tnorm\) is called a \emph{t-norm}.
\item A function $\imp: \left[0,1\right]^2\rightarrow \left[0,1\right]$ is called an \emph{implicator} if $\imp(0,0)=\imp(0,1)=\imp(1,1)=1$, \(\imp(1,0)=0\) and for all $x_1,x_2,y_1,y_2$ in $ \left[0,1\right]$ the following holds:
\begin{enumerate}
	\item $x_1\leq x_2\Rightarrow \imp(x_1,y_1)\geq \imp(x_2,y_1)$ (decreasing in the first argument),
	\item $y_1\leq y_2\Rightarrow \imp(x_1,y_1)\leq \imp(x_1,y_2)$ (increasing in the second argument),
\end{enumerate}
\item A function \(\mathcal{N}:[0,1]\to [0,1]\) is called a \emph{negator} if it is non-increasing and satisfies \(\mathcal{N}(0)=1\) and \(\mathcal{N}(1)=0\).
\item The \emph{induced conjunctor} of an implicator \(\mathcal{I}\) and negator \(\mathcal{N}\) is defined by:
\[\mathcal{C}_{\mathcal{I},\mathcal{N}}(x,y):=\mathcal{N}\left(\imp(x,\mathcal{N}(y))\right), \;\;\forall x,y\in [0,1].\]
\end{itemize}
\end{defn}
Since t-norms are required to be associative, they can be extended naturally to a function \([0,1]^n\to [0,1]\) for any natural number \(n\geq 2\).
\begin{exmp}
	\hfill
\begin{itemize}
\item The minimum and product operators are t-norms: \(\mathcal{T}_{M}(x,y):=\min(x,y)\) and \(\mathcal{T}_{P}(x,y):=x\cdot y\).
\item The Kleene-Dienes, Reichenbach and \L ukasiewicz implicators are defined by \(\mathcal{I}_{KD}(x,y):=\max(1-x,y)\), \(\imp_{R}(x,y):=1-x +x\cdot y\) and \(\imp_{L}(x,y):= \min(1-x+y,1)\).
\item The standard negator \(\mathcal{N}_s\) is defined by \(\mathcal{N}_s(x):= 1-x\) for \(x\in[0,1]\).
\item The induced conjunctor of the Kleene-Dienes implicator and the standard negator is the minimum t-norm.
\end{itemize}
\end{exmp}
\begin{defn}
Given a negator $\mathcal{N}$, the \(\mathcal{N}\)-complement of a fuzzy set is given by:
\[\comp_{\mathcal{N}}(A)(x)= \mathcal{N}(A(x)), \;\;\forall x\in X\]
\end{defn}
\subsection{Fuzzy rough sets}
Rough sets, first introduced by Pawlak \cite{pawlak1982rough}, try to model uncertainty that is associated with \emph{indiscernibility}. Here indiscernibility is defined with respect to an equivalence relation, and two elements are called \emph{indiscernible} if they are in the same equivalence class. Indiscernibility arises naturally in information systems.
\begin{defn}
An \emph{information system} $\left( X, \mathcal{A}\right)$, consists of a finite non-empty set \(X\) and a non-empty family of attributes $\mathcal{A}$, where each attribute \(a\in \mathcal{A}\) is a function $a:\ X \rightarrow V_a$, with $V_a$ the set of values the attribute $a$ can take.
A \emph{decision system} is an information system $\left( X, \mathcal{A}\cup \{d\}\right)$, where \(d\notin \mathcal{A}\) is called the \emph{decision attribute} and each \(a\in\mathcal{A}\) is called a \emph{conditional attribute}.
\end{defn}
\begin{defn}[$B$-indiscernibility]
	Let $B$ be a subset of $\mathcal{A}$, then the \emph{$B$-indiscernibility} relation is given by
	\begin{equation*}
		R_B = \left\{\left.\left(x,y\right)\in X^2\right| \forall a\in B, a(x) = a(y)\right\}.
	\end{equation*}
	If $\left(x,y\right)\in R_B$ then $x$ and $y$ are said to be \emph{indiscernible} with respect to $B$. The equivalence class of \(x\) is denoted by $\left[x\right]_B$.
\end{defn}
\begin{defn}\cite{pawlak1982rough}
\label{defintionRoughset}
Let $A$ be a subset of $X$ and \(R\) an equivalence relation on \(X\). The \emph{lower} and \emph{upper approximations} of $A$ with respect to \(R\) are defined as:
\begin{align*}
\underline{apr}_{R} A &= \left\{\left.x\in X\right| \left[x\right]_R \subseteq A \right\}=\left\{\left.x\in X\right| (\forall y\in X) \left((x,y)\in R \implies y\in A\right)\right\}\\
\overline{apr}_{R} A &= \left\{\left.x\in X\right| \left[x\right]_R \cap A \neq \emptyset \right\}=\left\{\left.x\in X\right| (\exists y\in X) \left((x,y)\in R \land y\in A\right)\right\},
\end{align*}
where \([x]_R:=\{y\in X | (x,y)\in R\}\) denotes the equivalence class of \(x\) with respect to \(R\).
The pair $\left(\underline{apr}_{R} A,\overline{apr}_{R} A\right)$ is called a \emph{rough} set. 
\end{defn}
If an element \(x\) is in the lower approximation of \(A\), then we know that all elements that are indiscernible from \(x\) are also in \(A\). When an element \(x\) is in the upper approximation of \(A\), we can only say that there exists some element that is indiscernible from \(x\) and belongs to \(A\). \\
For fuzzy sets and fuzzy relations, the lower and upper approximations can be extended as follows:
\begin{defn}{\cite{radzikowska2002comparative}}
	Given $R\in \mathcal{F}(X\times X)$ and $A\in\mathcal{F}(X)$, the \emph{lower} and \emph{upper approximation} of $A$ w.r.t.\ $R$ are defined as:
	\begin{align}
	(\underline{\text{apr}}_{R} A)(x) &= \min\limits_{y\in X} \imp(R(x,y),A(y)), \label{lower approx}\\
	(\overline{\text{apr}}_{R} A)(x) &= \max\limits_{y\in X} \mathcal{C}(R(x,y),A(y)),\label{upper approx}
	\end{align}
	where $\imp$ is an implicator and $\mathcal{C}$ a conjunctor.
\end{defn}
\subsection{OWA-based fuzzy rough sets}
A downside to the classical definition of lower and upper approximation in fuzzy rough set theory is their lack of robustness. The value of the membership of an element in the lower and upper approximation is fully determined by a single element because of the minimum and maximum operators in the definition. To solve this undesirable behaviour, OWA-based fuzzy rough sets were introduced in \cite{cornelis2010ordered}. The Ordered Weighted Average \cite{yagerOWA} is an aggregation operator that is defined as follows:
\begin{defn}[OWA operator]
	\label{OWA}
	Let \(X=\{x_1,x_2,\dots,x_n\}\), \(f:X\to \mathbb{R}\) and \(\mathbf{w}=(w_1,w_2,\dots,w_n)\in [0,1]^n\) be a weighting vector, i.e.\ \(\sum_{i=1}^{n}w_i=1\), then the \emph{ordered weighted average} of \(f\) with respect to \(\mathbf{w}\) is defined as
	\[\text{OWA}_{\mathbf{w}}(f):=\sum_{i=1}^n f(x_{\sigma(i)})w_i,\]
	where $\sigma$ is a permutation of \(\{1,2,\dots,n\}\) such that
	\begin{equation*}
	f(x_{\sigma(1)})\geq f(x_{\sigma(2)}) \geq\cdots\geq f(x_{\sigma(n)}).
	\end{equation*}
\end{defn}
\begin{exmp}
	\label{OWAminMaxAverage}
	The maximum, mean and minimum operators can all be seen as OWA-operators with weight vectors \((1,0,\dots,0,0)\), \(\left(\frac{1}{n},\frac{1}{n},\dots,\frac{1}{n}\right)\) and \((0,0,\dots,0,1)\) respectively.
\end{exmp}
In OWA-based fuzzy rough sets, OWA operators replace the minimum and maximum in equations (\ref{lower approx}) and (\ref{upper approx}). To not deviate too strongly from the original definitions, some requirements may be enforced on the weight vectors of the OWA-operators used \cite{cornelis2010ordered}. In particular, we require that the OWA-operator for the lower approximation is a soft minimum and for the upper approximation a soft maximum.
\begin{defn}
	\label{ornessOWA}
	The \emph{orness} and \emph{andness} of a weight vector \(\mathbf{w}=(w_i)_{i=1}^n\) are defined as
	\begin{align}
	\orness(\mathbf{w})&=\frac{1}{n-1}\sum_{i=1}^{n}((n-i)\cdot w_i), \label{orness}\\
	\andness(\mathbf{w})&= 1- \orness(\mathbf{w}).\nonumber
	\end{align}
	If \(\orness(\mathbf{w})<0.5\), then \(OWA_\mathbf{w}\) is called a \emph{soft minimum}. If \(\orness(\mathbf{w})>0.5\), \(OWA_\mathbf{w}\) is called a \emph{soft maximum}.
\end{defn}
As can be seen from Equation (\ref{orness}), the orness indicates how much weight is given to the largest elements. The orness tells us how ``close'' the OWA-operator is to the maximum. Using this definition OWA-based fuzzy rough sets are then defined as:
\begin{defn} \cite{cornelis2010ordered}
	\label{owafuzzyrough}
	Given $R\in\mathcal{F}(X\times X)$, weight vectors \(\mathbf{w}_l\) and \(\mathbf{w}_u\) with \(\orness(\mathbf{w}_l)<0.5\) and \(\orness(\mathbf{w}_u)>0.5\) and \(A\in\mathcal{F}(X)\), the \emph{OWA lower} and \emph{upper approximation} of $A$ w.r.t.\ $R$, $\mathbf{w}_l$ and \(\mathbf{w}_u\) are given by:
	\begin{align}
	(\underline{\text{apr}}_{R,\mathbf{w}_l}A)(x)&= OWA_{\mathbf{w}_l}\left( \imp(R(x,y),A(y))\right),\\
	(\overline{\text{apr}}_{R,\mathbf{w}_u}A)(x) &= OWA_{\mathbf{w}_u}\left( \mathcal{C}(R(x,y),A(y))\right),
	\end{align}
	where $\imp$ is an implicator, $\mathcal{C}$ a conjunctor and \(\imp(R(x,y),A(y))\) and \(\mathcal{C}(R(x,y),A(y))\) are seen as functions in \(y\).
\end{defn}
\subsection{The Choquet integral}
The Choquet integral induces a large class of aggregation functions, namely the class of all comonotone linear aggregation functions \cite{beliakov2007aggregation}.
Since we will view the Choquet integral as an aggregation operator, we will restrict ourselves to measures (and Choquet integrals) on finite sets. For the general setting, we refer the reader to e.g.\ \cite{wang2010generalized}.
\begin{defn}
	A function \(\mu:\mathcal{P}(X)\to[0,1]\) is called a \emph{monotone measure} if:
	\begin{itemize}
		\item $\mu(\emptyset)=0$ and \(\mu(X)=1\)
		\item \((\forall A,B\in(\mathcal{P}(X))(A\subseteq B \implies \mu(A)\leq \mu(B))\)
	\end{itemize}
	A monotone measure is called:
	\begin{itemize}
		\item  \emph{additive} if \(\mu(A\cup B)=\mu(A)+\mu(B)\) when \(A\) and \(B\) are disjoint

		\item \emph{symmetric} if \(\mu(A)=\mu(B)\) when \(\abs{A}=\abs{B}\)
	\end{itemize}
\end{defn}
\begin{defn}\cite{wang2010generalized}
	\label{defn: ChoquetIntegral}
	Let $\mu$ be a monotone measure on \(X\) and \(f:X\to\mathbb{R}\) a real-valued function. The \emph{Choquet integral} of \(f\) with respect to the measure $\mu$ is defined as:
	\begin{equation*}
	\int f \diff \mu=\sum_{i=1}^{n}\mu(A^\ast_i)\cdot\left[f(x^\ast_i)-f(x^\ast_{i-1})\right],
	\end{equation*}
	where \((x^\ast_1,x^\ast_2,\dots,x^\ast_n)\) is a permutation of \(X=(x_1,x_2,\dots,x_n)\) such that
	\begin{equation*}
	f(x^\ast_1)\leq f(x^\ast_2) \leq\cdots\leq f(x^\ast_n),
	\end{equation*}
	\(A^\ast_i:=\{x^\ast_i,\dots,x^\ast_n\}\) and \(f(x^\ast_0):=0\).
\end{defn}
The following proposition gives an equivalent definition of the Choquet integral:
\begin{prop}\cite{wang2010generalized}
Let $\mu$ be a monotone measure on \(X\), \(f:X\to\mathbb{R}\) a real-valued function. Then the following holds (using the notation of Definition \ref{defn: ChoquetIntegral}):
\begin{equation*}
\int f \diff \mu=\sum_{i=1}^n f(x^\ast_i)\cdot\left[\mu(A^\ast_i)-\mu(A^\ast_{i+1})\right],
\end{equation*}
where \(\mu(A^\ast_{n+1}):=0\).
\end{prop}
\begin{cor}{\cite{wang2010generalized}}
\label{cor:equivalentChoquet}
The following equality holds for every monotone measure \(\mu\), constant \(c\in\mathbb{R}\) and real-valued function \(f:X\to \mathbb{R}\):
\[\int(c+f)\diff \mu= c + \int f \diff \mu.\]
\end{cor}
The class of aggregation operators induced by the Choquet integral contains the weighted mean and the OWA operator. In fact, the weighted mean and OWA operator are the Choquet integrals with respect to additive and symmetric measures, respectively.
\begin{prop}{\cite{beliakov2007aggregation}}
	The Choquet integral with respect to an additive measure \(\mu\) is the weighted mean \(M_\mathbf{w}\) with weight vector \(\mathbf{w}=(w_i)_{i=1}^n=(\mu(\{x_i\}))_{i=1}^n\). Conversely, the weighted mean \(M_\mathbf{v}\), with weight vector \(\mathbf{v}=(v_i)_{i=1}^n\) is a Choquet integral with respect to the uniquely defined additive measure \(\mu\) for which \((\mu(\{x_i\}))_{i=1}^n=(v_i)_{i=1}^n\).
\end{prop}

\begin{prop}{\cite{beliakov2007aggregation}}
	\label{ChoquetOWA}
	The Choquet integral with respect to a symmetric measure \(\mu\) is the OWA operator with weight vector \(\mathbf{w}=(w_i)_{i=1}^n=(\mu(A_{i})-\mu(A_{i-1}))_{i=1}^n\), where \(A_i\) denotes any subset with cardinality \(i\). Conversely, the OWA operator with weight vector \(\mathbf{v}=(v_i)_{i=1}^n\) is a Choquet integral with respect to the symmetric measure \(\mu\) defined as \[(\forall A\subseteq X)(\mu(A):= \sum_{i=1}^{\abs{A}}v_i).\]
\end{prop}
Lastly we recall the definition of a dual measure and take a look at how this concept translates to the case of OWA operators.
\begin{defn}{\cite{wang2010generalized}}
	The dual measure $\overline{\mu}$ of a monotone measure $\mu$ is defined by:
	\[\overline{\mu}(A)=\mathcal{N}_s(\mu(\comp_{\mathcal{N}_s} A)).\]
\end{defn}
\begin{prop}{\cite{wang2010generalized}}
	\label{lemmaDuality}
	Let $\mu$ be a monotone measure on \(X\) and \(f\) a real-valued function on \(X\). Then
	\[\int f \diff \overline{\mu}= \mathcal{N}_s\left(\int \mathcal{N}_s(f)\diff \mu\right)\]
\end{prop}
\begin{prop}
\label{dualityOWA}
If $\mu$ is the symmetric measure corresponding to the OWA operator with weights \(\mathbf{w}=(w_i)_{i=1}^n\), then \(\overline{\mu}\) corresponds to the OWA operator with weights \(\overline{\mathbf{w}}=(w_{n-i+1})_{i=1}^n\).
\end{prop}
\begin{proof}
From Proposition \ref{ChoquetOWA} we know that
\[\mu(A):= \sum_{i=1}^{\abs{A}}w_i,\]
and thus
\begin{align*}
\overline{\mu}(A)&=1-\sum_{i=1}^{n-\abs{A}}w_i=\sum_{i=1}^n w_i-\sum_{i=1}^{n-\abs{A}}w_i\\
&=\sum_{i=n-\abs{A}+1}^{n}w_i=\sum_{i=1}^{\abs{A}}w_{n-i+1}\\
&=\sum_{i=1}^{\abs{A}}\overline{w}_{i},
\end{align*}
which using Proposition \ref{ChoquetOWA} proves the proposition.
\end{proof}
\section{Vague quantification}
\label{sec: vague quantification}
In this section, we recall the models of Zadeh \cite{zadehFuzzyQuantifier} and Yager \cite{yager1996quantifier} for vague quantification and we show how particular fuzzy quantifiers correspond to symmetric measures. Vague quantification models linguistic quantifiers such as ``most'', ``some'', ``almost all'', etc.
\subsection{Zadeh's model of vague quantification}
Zadeh's model represents a proportional linguistic quantifier as a fuzzy set \(Q\) of the unit interval. If \(p\) is the proportion for which a certain proposition holds, then \(Q(p)\) determines the truth value of the quantified proposition.
\begin{defn}{\cite{zadehFuzzyQuantifier}}
A fuzzy set \(Q\in\mathcal{F}([0,1])\) is called a \emph{regular increasing monotone (RIM) quantifier} if
\(Q\) is a non-decreasing function such that \(Q(0)=0\) and \(Q(1)=1\).
\end{defn}
\begin{exmp}
\label{exmp: RIM quantifiers}
The following RIM quantifiers represent the universal and existential quantifier:
\begin{align*}
Q_\forall(p)=\left\{
\begin{array}{ll}
	1 &\text{ if } p = 1\\
	0 & \text{ if } p < 1
\end{array}
\right.\;
\;\;Q_\exists(p)=\left\{
\begin{array}{ll}
	1 &\text{ if } p\neq 0\\
	0 & \text{ if } p=0
\end{array}
\right..
\end{align*}
Linguistic quantifiers such as ``most'' and ``some'' can be modelled using
the RIM quantifiers \(Q_{(\alpha,\beta)}\) (\(0\leq \alpha < \beta \leq 1\)) \cite{cornelis2007vaguely}:
\begin{align*}
	Q_{(\alpha,\beta)} (p)&= \left\{
	\begin{array}{ll}
		0 &  \;p\leq \alpha\\
		\frac{2(p-\alpha)^2}{(\beta-\alpha)^2} & \; \alpha \leq p \leq \frac{\alpha+\beta}{2}\\
		1-\frac{2(p-\beta)^2}{(\beta-\alpha)^2}  &  \;  \frac{\alpha+\beta}{2}\leq p\leq \beta \\
		1 & \; \beta\leq p
	\end{array}
	\right.,
\end{align*}
for example, we could use \(Q_{(0.3, 0.9)}\) and \(Q_{(0.1, 0.4)}\) to model ``most'' and ``some'', respectively.
\end{exmp}
Let \(Q\) be a RIM quantifier, Zadeh then uses \(Q\) to evaluate the truth value of the proposition \(``Q X's \text{ are } A's\text{''}\) as
\begin{align*}
Q\left(\frac{\abs{A}}{\abs{X}}\right),
\end{align*}
where \(A\in\mathcal{F}(X)\) and the cardinality of a fuzzy set is interpreted using Zadeh's $\Sigma$-count, i.e.\ \(\abs{A}=\sum_{x\in X}A(x)\). 
\subsection{Yager's model of vague quantification}
Yager's approach also represents quantifiers as fuzzy sets of the unit interval, but uses OWA aggregation for the evaluation of the truth values. He evaluates propositions of the form ``\(QX\)'s are \(A\)'s'', where \(Q\) is a RIM quantifier and \(A\) is a fuzzy set, as
\begin{equation}
	\label{yagerOWA}
OWA_{\mathbf{w}}(A), \;\; \text{where}\;\; w_i := Q\left(\frac{i}{n}\right)-Q\left(\frac{i-1}{n}\right).
\end{equation}
\begin{exmp}
	Suppose \(X=\{x_1,x_2, x_3, x_4\}\) is the set of all basketball players and \(A=\{(x_1, 0.5),(x_2,0.5),(x_3,1),(x_4,1)\}\) the fuzzy set describing tallness. Then the truth value of the statement ``most basketball players are tall'', in Yager's model, is given by:
	\begin{align*}
		OWA_{\mathbf{w}}(A)&= w_1 + w_2 + 0.5* (w_3+w_4)\\
		&\approx 0 + 0.22 + 0.5*(0.65+0.125) \approx 0.61,
	\end{align*}
where Equation \eqref{yagerOWA} is used with \(Q_{(0.3,0.9)}\).\\
In Zadeh's model this evaluates to:
\begin{align*}
	Q_{(0.3, 0.9)}\left(\frac{\abs{A}}{\abs{X}}\right)&= Q_{(0.3, 0.9)}\left(\frac{\sum_{x\in X} A(x)}{4}\right)\\
	&= Q_{(0.3, 0.9)}\left(\frac{0.5+0.5+1+1}{4}\right) = 0.875.
\end{align*}
Notice that the truth value in Yager's model is closer to what one might expect, since half of the basketball players only have a membership degree of \(0.5\) to the fuzzy set ``tall people''. The reason for this is that Zadeh's model only looks at the average membership degree, whereas Yager's model gives weight to each instance according to its relative degree of membership.
\end{exmp}
The next definition and proposition describe the relationship between weight vectors and RIM quantifiers.
\begin{defn}
	A RIM quantifier \(Q\) is associated with a weight vector \(\mathbf{w}\) if \(Q\) interpolates the following set of points:
	\begin{equation}
		\label{quantifierAssWeight}
\bigcup_{i=1}^n\left\{\left(\frac{i}{n}\,,\,\sum_{j\leq i}w_j\right)\right\}.
	\end{equation}
Two RIM quantifiers are called \emph{semantically equivalent} (on a universe \(X\)) if they are associated with the same weight vector.
\end{defn}
\begin{prop}
	A RIM quantifier \(Q\) is associated with a weight vector \(\mathbf{w}\) if and only if
	\begin{equation}
		\label{weightAssQuantifier}
w_i = Q\left(\frac{i}{n}\right)-Q\left(\frac{i-1}{n}\right),
	\end{equation}
	for every \(i\in\{1,\dots,n\}\).
\end{prop}
\begin{proof}
	Suppose \(Q\) is a RIM quantifier that interpolates the points
	\[\bigcup_{i=1}^n\left\{\left(\frac{i}{n}\,,\,\sum_{j\leq i}w_j\right)\right\},\]
	then
	\[Q\left(\frac{i}{n}\right)-Q\left(\frac{i-1}{n}\right)=\sum_{j\leq i}w_j-\sum_{j\leq i-1}w_j=w_i.\]
	Conversely, suppose 
	\[w_i = Q\left(\frac{i}{n}\right)-Q\left(\frac{i-1}{n}\right),\]
	for every \(i\in\{1,\dots,n\}\). By induction we have that 
	\[Q\left(\frac{i}{n}\right)= w_i+Q\left(\frac{i-1}{n}\right)=w_i+\sum_{j\leq i-1}w_j=\sum_{j\leq i}w_j.\]
\end{proof}
\begin{cor}
Semantically equivalent RIM quantifiers have the same truth value for propositions of the form ``\(QX\)'s are \(A\)'s'' (Equation \eqref{yagerOWA}).
\end{cor}
The previous proposition gives us a one-to-one correspondence (modulo RIM quantifiers with equivalent semantics) between RIM quantifiers and OWA weights, when the universe \(X\) is fixed. Indeed, for every OWA weight vector \(\mathbf{w}\), we can define \(Q\) as the unique step function that interpolates the set of points defined in Equation \eqref{quantifierAssWeight}:
\[Q(p)= \sum_{i\leq p*n}w_i,\]
 which is indeed a RIM quantifier since \(\mathbf{w}\) is a weight vector. Conversely, we have Equation \eqref{weightAssQuantifier}.
\begin{exmp}
	\label{exmp: additive quantifier}
	A common weighting scheme (cf.\ \cite{vluymans2019dealing}) for the OWA operator, in the context of fuzzy rough lower approximations, is the additive weight vector defined by 
	\[W_L^{add}=\left\langle \frac{2}{n(n+1)},\frac{4}{n(n+1)},\dots,\frac{2(n-1)}{n(n+1)},\frac{2}{n+1}\right\rangle.\]
	We now want to find a quantifier \(Q_{add}\) that is associated with \(W^{add}_L\). More specifically, \(Q_{add}\) has to interpolate the set of points
	\[\bigcup_{i=1}^n\left\{\left(\frac{i}{n}\,,\,\sum_{j\leq i}\frac{2j}{n(n+1)}\right)\right\}=\bigcup_{i=1}^n\left\{\left(\frac{i}{n}\,,\,\frac{2}{n(n+1)}\sum_{j\leq i}j\right)\right\}=\bigcup_{i=1}^n\left\{\left(\frac{i}{n}\,,\,\frac{i(i+1)}{n(n+1)}\right)\right\},\]
	so we get, by applying the substitution \(x=i/n\), that the following RIM quantifier \(Q_{add}\) corresponds to \(W^{add}_L\):
	\[Q_{add}(x)=\frac{x(xn+1)}{n+1}.\]
\end{exmp}
Since we know that weight vectors of OWA operators are equivalent to symmetric measures, we can also convert a RIM quantifier \(Q\) into a symmetric measure \(\mu_Q\) (Proposition \ref{ChoquetOWA}):
\begin{align}
	\label{vagueQuantif}
	\mu_{Q}(A)=Q\left(\frac{\abs{A}}{\abs{X}}\right),
\end{align}
and we can rewrite the evaluation as follows:
\begin{equation}
	\label{Eq: vagueChoquet}
``Q X's \text{ are } A's\text{''} \text{ evaluated as }	\int A \diff \mu_Q,
\end{equation}
which in the case that \(A\) is crisp reduces to \(\mu_Q(A)\) (follows from Definition \ref{defn: ChoquetIntegral}). When \(A\) is crisp, Yager's and Zadeh's models thus coincide.
\section{Choquet-based fuzzy rough sets}
\label{choq-basedFuzzyRoughSets_section}
\subsection{Motivation and definition}
Note that by Proposition \ref{ChoquetOWA}, we can rewrite OWAFRS as follows:
	\begin{align}
		\label{lowerapprox}
	(\underline{\text{apr}}_{R,\mu_l}A)(y)&=\int \mathcal{I}(R(x,y),A(x))\diff\mu_l(x),\\
	\label{upperapprox}
	(\overline{\text{apr}}_{R,\mu_u}A)(y)&=\int
	\mathcal{C}(R(x,y),A(x))\diff\mu_u(x),
\end{align}
where $\mu_l$ and $\mu_u$ are two symmetric measures. From the previous section we know that these expressions can be interpreted as vaguely quantified propositions. For example, suppose \(\mu_l\) is the measure corresponding with ``most'' and $\mu_u$ with ``some'', then the degree of membership of an element \(y\) to the lower approximation (Equation \eqref{lowerapprox}) is equal to the truth value of the proposition ``Most elements indiscernible to \(y\) are in \(A\)''. Analogously we have that the degree of membership of an element \(y\) to the upper approximation (Equation \eqref{upperapprox}) is equal to the truth value of the proposition ``Some elements are indiscernible to \(y\) and are in \(A\)''.  This approach is thus intuitive, since it closely resembles the definition of rough sets. But why restrict ourselves to symmetric measures? If we allow non-symmetric measures, we gain more flexibility to reduce noise, as the following example shows.
\begin{exmp}
Suppose we have a crisp set \(O\) containing all the instances that are outliers, unreliable or inaccurate, then a useful pair of quantifiers could be ``for all except (maybe) elements of \(O\)'' and ``there exists an element in \(X \setminus O\)''. These quantifiers can be modelled by the partial minimum and maximum:
\begin{align*}
\label{ForallXwithoutO}
\mu_{\forall x\in X\setminus O}(B)&=\left\{
\begin{array}{ll}
	1 &\text{ if } X\bslash O \subseteq B\\
	0 & \text{elsewhere}
\end{array}
\right.,\\
\mu_{\exists x\in X\bslash O}(B)&:=\overline{\mu}_{\forall x\in X\setminus O}(B)=1-\mu_{\forall x\in X\setminus O}(\comp B)\\&=\left\{
\begin{array}{ll}
	0 &\text{ if } X\bslash O \subseteq X\bslash B\\
	1 & \text{elsewhere}
\end{array}
\right.=\left\{
\begin{array}{ll}
	0 &\text{ if } B\subseteq O\\
	1 & \text{elsewhere}
\end{array}
\right..
\end{align*}
\end{exmp}
\begin{prop}
	\label{propPartialMinimum}
	Let \(O\) be a crisp subset of \(X\) and \(f:X\to \mathbb{R}\), then
	\[\int f \diff \mu_{\forall x\in X\setminus O}=\min_{x\in X\setminus O}f(x).\]
\end{prop}
\begin{proof}
	Let \(\left( x^\ast_1,x^\ast_2,\dots,x^\ast_n\right)\) be an ordering of \(X\) such that the following inequalities hold:
	\[f(x^\ast_1)\leq \dots\leq f(x^\ast_k)\leq \dots \leq f(x^\ast_n),\]
	where \(k\) is the first index such that \(x^\ast_k\in X\setminus O\). Using the definition of the Choquet integral we get:
	\begin{align*}
		\int f \diff \mu_{\forall x\in X\setminus O}&=\sum_{i=1}^{n}\left[f(x^\ast_i)-f(x^\ast_{i-1})\right]\cdot\mu_{\forall x\in X\setminus O}(\{x^\ast_i,x^\ast_{i+1},\dots,x^\ast_n\})\\
		&=\sum_{i=1}^{k}\left[f(x^\ast_i)-f(x^\ast_{i-1})\right]\cdot 1+\sum_{i=k+1}^{n}\left[f(x^\ast_i)-f(x^\ast_{i-1})\right]\cdot0\\
		&=f(x^\ast_{k})=\min_{x\in X\setminus O}f(x).
	\end{align*}
\end{proof}
\begin{cor}
	\label{propPartialMaximum}
	Let \(O\) be a crisp subset of \(X\) and \(f:X\to \mathbb{R}\), then 
	\[\int f \diff \mu_{\exists x\in X\setminus O}=\sup_{x\in X\setminus O}f(x).\]
\end{cor}
\begin{proof}
	Directly follows from Proposition \ref{lemmaDuality}:
	\begin{align*}
		\int f \diff \mu_{\exists x\in X\setminus O}&=\int f \diff \overline{\mu}_{\forall x\in X\setminus O}=-\int (-f) \diff \mu_{\forall x\in X\setminus O}\\&=- \inf_{x\in X\setminus O}\left(-f(x)\right)=\sup_{x\in X\setminus O}f(x).
	\end{align*}
\end{proof}
Using these non-symmetric measures in Equation \eqref{lowerapprox} and \eqref{upperapprox}, we get that the degree of membership of an element \(y\) to the lower approximation is equal to the truth value of the proposition ``All trustworthy elements that are indiscernible to \(y\) are in \(A\)''. An analogous interpretation holds for the upper approximation.\\
As we will show in Subsection \ref{section:measures}, it is possible to extend the approach of the previous example to fuzzy sets \(O\) and quantifiers representing "most of the trustworthy objects". The following examples show how such fuzzy sets \(O\) can be constructed in practice.
\begin{exmp}
\label{exmp: outlierscoreO}
Suppose we have a decision system \((X,\mathcal{A}\cup\{d\})\) where \(d\) is a categorical attribute. Then we can define \(O(x)\) as the normalized outlier score \cite{normalizedLOF} of \(x\) (obtained from a certain outlier detection algorithm) when compared to other elements of \([x]_d\) (based on the conditional attributes). An outlier score measures the degree to which a data point differs from other observations, and normalization transforms this score in such a way that it can be interpreted as a degree of outlierness.
\end{exmp}
\begin{exmp}
\label{exmp: patientsO}
Suppose \(X\) consists of patients from several different hospitals, \(A\) is the subset of patients that have a disease and \(R\) is a similarity relation between patients based on a set of symptoms. Then a confidence score \(c_i\) can be attached to each hospital \(i\) based on the accuracy of the tests performed to trace the disease (and the symptoms). The membership degree of a patient \(x\) of hospital \(i\) to \(O\) can then be defined as \(O(x)=1-c_i\). 
\end{exmp}
These examples motivate the definition of Choquet-based fuzzy rough sets (CFRS):
\begin{defn}
	\label{defn: Choquet-basedfuzzyRoughSets}
	Given \(R\in\mathcal{F}(X\times X)\), monotone measures \(\mu_l\) and \(\mu_u\) on \(X\) and $A\in\mathcal{F}(X)$, then the \emph{Choquet lower} and \emph{upper approximation} of $A$ w.r.t.\ $R$, $\mu_l$ and \(\mu_u\) are given by:
	\begin{align*}
	(\underline{\text{apr}}_{R,\mu_l}A)(y)&=\int \mathcal{I}(R(x,y),A(x))\diff\mu_l(x)\\
	(\overline{\text{apr}}_{R,\mu_u}A)(y)&=\int
	\mathcal{C}(R(x,y),A(x))\diff\mu_u(x),
	\end{align*}
	where $\imp$ is an implicator and $\mathcal{C}$ is a conjunctor.
\end{defn}
Notice that we have discarded the orness conditions in the definition of OWA-based fuzzy rough sets. The reason for this is that these orness conditions were introduced mainly in the hope they would yield extra theoretical properties such as inclusion of the lower approximation in the upper approximation, but as we shall see in Subsection \ref{sec:theoreticalproperties}, they do not.
\subsection{Examples of non-symmetric measures}
\label{section:measures}
As described in the previous subsection we can accommodate non-symmetry by introducing a fuzzy set \(O\) in \(X\) that represents the degree of inconfidence. The function \(O(x)\) could, for example, be seen as an outlier score in \([0,1]\) (Example \ref{exmp: outlierscoreO}) or it could represent the unreliability or inaccuracy of the observation (Example \ref{exmp: patientsO}).  We now define several non-symmetric measures using the fuzzy set \(O\).
\subsubsection{Fuzzy removal}
One option to use \(O\) to define a non-symmetric measure is as follows:
\begin{equation}
\label{fuzzy removal}
\mu_{\forall x\in X\bslash O}(A)=\left\{
\begin{array}{ll}
1&\text{ if }A=X \\
0 &\text{ if }A=\emptyset\\
\mathcal{T}\underbrace{(O(x))}_{x\in X\setminus A} & \text{elsewhere}
\end{array}
\right.,
\end{equation}
where \(\mathcal{T}\) is a t-norm (e.g.\ minimum).
\begin{prop}
 The function \(\mu_{\forall x\in X\bslash O}\) is a monotone measure.
\end{prop}
\begin{proof}
The monotonicity follows from the following property of t-norms:
\[\mathcal{T}(x_1,x_2,\dots,x_k)\geq \mathcal{T}(x_1,x_2,\dots,x_n)\;\text{ if }k\leq n\text{ and } x_i\in[0,1],\]
which can be seen using induction, the neutral element property and the increasingness of t-norms. 
\end{proof}
 We will call this measure the \emph{fuzzy removal} measure, since in the case \(O\) is crisp, the Choquet integral with respect to this measure is equal to the partial minimum (Proposition \ref{propPartialMinimum})
The vague quantifier interpretation of the fuzzy removal measure could thus be ``for all except (maybe) elements of \(O\)''.
\begin{exmp}
\label{exmp:fuzzy removal}
Let \(X= \{x_1,x_2,x_3,x_4,x_5\}\) be the set of people who went to a party and \(A= \{x_1, x_2, x_3\}\) the subset of people who tested positive for a certain disease. Now suppose we know that \(x_1, x_2\) and \(x_3\) went to the hospital for the test, while \(x_4\) and \(x_5\) used a home testing kit. Then a sensible choice for \(O\), which in this case represents the unreliability,  could be \(O = \{(x_1,0), (x_2, 0), (x_3,0), (x_4,0.3), (x_5, 0.3)\}\), since \(x_4\) and \(x_5\) can report any result they want and the tests are less accurate. Then the evaluation of the sentence ``all reliable tests from people at the party are positive'' is:
\begin{align*}
\int A \diff \mu_{\forall x\in X\setminus O} = \mu_{\forall x\in X\setminus O} (A) = \tnorm_{\min}(O(x_4),O(x_5))= 0.3,
\end{align*}
it thus evaluates to the t-norm of all the unreliabilities of the elements that are not in \(A\). Thus if there was one element \(x\) fully reliable (\(O(x)=0\)) and with a negative test, then the sentence would evaluate to false, as expected.
\end{exmp}
\subsubsection{Weighted ordered weighted  average}
Another idea for a non-symmetric measure is:
\begin{equation}
	\label{WOWA}
	\mu(A)=Q\left(\sum_{x_{i}\in A} p_i\right),
\end{equation}
where \(Q\) is a RIM quantifier and \(\mathbf{p}\) a weight vector describing the confidence, reliability, accuracy or non-outlierness of each observation:
\begin{equation}
	\label{WOWA_weight}
	p_i =  \frac{1-O(x_i)}{n-\sum_{j=1}^{n}O(x_j)}.
\end{equation}
The measure in Equation \eqref{WOWA} corresponds to the Weighted Ordered Weighted Averaging (WOWA) operator \cite{WOWA,WOWA<Choquet}, which is a generalization of the OWA and the weighted mean. The RIM quantifier \(Q\) determines the OWA part of the WOWA and the weight vector \(\mathbf{p}\) the weighted mean part. The WOWA operator is also equivalent with Yager's importance weighted quantifier guided aggregation \cite{yager1996quantifier}. These measures could be interpreted as quantifiers of the form ``\(Q\) of the trustworthy/reliable objects''. 
\begin{exmp}
\label{exmp: wowa}
Recall Example \ref{exmp:fuzzy removal}, and suppose we want to evaluate the sentence ``Most reliable tests from people at the party are positive''. Then a suitable choice for the measure would be Equation \eqref{WOWA} together with Equation \eqref{WOWA_weight} and the RIM quantifier ``most'' from Example \ref{exmp: RIM quantifiers}. The sentence can then be evaluated as:
\begin{align*}
\int A \diff \mu = \mu (A) &= Q_{(0.3,0.9)}\left(\sum_{x_i \in A} p_i\right)=Q_{(0.3,0.9)}\left(\frac{1 + 1 + 1}{5- 0.3-0.3}\right)\\
&\approx Q_{(0.3,0.9)}\left(0.68\right) \approx 0.875.
\end{align*}
If we do not take into account reliability, i.e.\ we evaluate the sentence ``Most people at the party are tested positive'' we get:
\[\int A \diff \mu_{Q_{(0.3,0.9)}}= Q_{(0.3,0.9)}\left(\frac{3}{5}\right)= 0.5,\]
which is, as expected, smaller than when considering reliability.
\end{exmp}
Another WOWA operator that could be used goes as follows: suppose $\pi$ is a permutation on \(X\) such that the following inequalities hold:
\[O(x_{\pi(1)})\leq O(x_{\pi(2)})\leq\dots\leq O(x_{\pi(n)}).\]
Using this ordering we can define a measure \(\mu\) as:
\begin{equation}
\label{exmp2)symmetric1}
\mu(A)=Q\left(\sum_{x_{\pi(i)}\in A} w_i\right),
\end{equation}
where \(Q\) is a RIM quantifier and $\mathbf{w}$ an \(n\)-dimensional weight vector. The weight \(w_i\) can be thought of as the weight given to the element with the \(i\)th smallest membership in \(O\).  If \(w_i=\frac{1}{n}\) for all \(i\in\{1,2,\dots,n\}\), we get the symmetric measure defined by \eqref{vagueQuantif}. Suppose \(O(x_{\pi(1)}),\dots,O(x_{\pi(k)})\) are all roughly zero and the rest roughly one, in other words we partitioned \(X\) in a set of non-outliers and a set of outliers (or accurate and inaccurate observations etc.). Then the following weight vector \(\mathbf{w}\) could be meaningful:
\begin{equation}
\label{exmp2-symmetric2}
w_i=\left\{
\begin{array}{ll}
\frac{(1-t)}{k}+\frac{t}{n} &\text{ if } i\in\{1,\dots,k\}\\
\frac{t}{n}& \text{elsewhere}
\end{array}
\right.,
\end{equation}
where \(t\in[0,1]\) is the weight given to the outliers. Notice that the measure from Equation \eqref{exmp2)symmetric1}, together with the weights from Equation \eqref{exmp2-symmetric2}, can be written as follows (when \(O\) is crisp):
\begin{equation}
	\label{TSmeasure}
	\mu(A)= Q\left(\abs{A\cap O}*\frac{t}{n} + \abs{A\cap \comp O}*\left(\frac{(1-t)}{k}+\frac{t}{n}\right)\right).
\end{equation}
The measure of a set \(A\) only depends on the cardinalities \(\abs{A\cap O}\) and \(\abs{A\cap \comp O}\), it is thus a two-symmetric measure \cite{p-symmetry}.
If \(t=0\) then \(\mu\) can be seen as a symmetric measure on the set \(X\setminus O\). On the other hand if \(t=1\), then we don't exclude outliers to any degree, and have a symmetric measure on \(X\). 
\begin{exmp}
Recall Example \ref{exmp: wowa} and let us define \(O\) as a crisp set \(O=\{x_4,x_5\}\). If we evaluate the sentence ``Most reliable tests from people at the party are positive'', using the measure from Equation \eqref{TSmeasure}, we get:
\begin{align*}
\int A \diff \mu = Q_{(0.3,0.9)}\left(3* \left(\frac{1-t}{3}+\frac{t}{5}\right)\right)=Q_{(0.3,0.9)}\left(1-t*\frac{2}{5}\right).
\end{align*}
\end{exmp}
\subsection{Theoretical properties of CFRS}
\label{sec:theoreticalproperties}
In this subsection, we take a look at the theoretical properties of CFRS. We start out with the following proposition that states that the only pair of measures for which the lower approximation is necessarily contained in the upper approximations, corresponds to the universal and existential quantifier.
\begin{prop}
	If \(\mu_l\) and \(\mu_u\) are two monotone measures such that
	\begin{equation}
		\label{eq:propinc}
		\underline{\text{apr}}_{R,\mu_l}A\subseteq \overline{\text{apr}}_{R,\mu_u}A
	\end{equation}
	holds for every equivalence relation \(R\) and crisp set \(A\), then \(\mu_l=\mu_\forall\) and \(\mu_u=\mu_\exists\).
\end{prop}
\begin{proof}
	Let \(B\subsetneq X\) be a non-empty subset, \(A\in \{\emptyset, X\}\),  \(\mu_l\) and \(\mu_u\) be two monotone measures satisfying Equation \ref{eq:propinc} and \(R\) the equivalence relation corresponding with the partition \(\{B , \comp B\}\) of \(X\). Then we have for \(y\in B\) and \(a\) the constant value of the set \(A\) that
	\begin{align*}
		I_{y}(x)&:=\mathcal{I}(R(x,y),A(x))=\left\{
		\begin{array}{ll}
			a,\;\; &x\in B \\
			1,\;\;\ &x \notin B
		\end{array}
		\right.,\\
		C_{y}(x)&:=\mathcal{C}(R(x,y),A(x))=\left\{
		\begin{array}{ll}
			a,\;\; &x\in B \\
			0,\;\;\ &x \notin B
		\end{array}
		\right.,
	\end{align*}
	since \(a\in \{0,1\}\) and implicators and conjunctors are extensions of their boolean counterparts. Calculating the lower and upper approximations using Definition \ref{defn: ChoquetIntegral} of the Choquet integral, we get:
	\begin{align*}
		(\underline{\text{apr}}_{R,\mu_l}A)(y)&=\int I_y(x)\diff\mu_l (x)=I_y(x^\ast_1)\mu_l(A^\ast_1)+\sum_{i=2}^n\mu(A^\ast_i)\cdot\left[I_y(x^\ast_i)-I_y(x^\ast_{i-1})\right]\\
		&= a + \mu_l(\comp B)\cdot (1-a),\\
		(\overline{\text{apr}}_{R,\mu_u}A)(y)&=\int C_y(x)\diff\mu_u (x)=\mu_u(B)\cdot a.
	\end{align*}
	Equation \ref{eq:propinc} gives us:
	\[a + \mu_l(\comp B)\cdot (1-a)\leq \mu_u(B)\cdot a,\]
	since this holds for an arbitrary non-empty \(B\subsetneq X\) and \(a\in \{0,1\}\) we have that \(\mu_l = \mu_\forall\) and \(\mu_u=\mu_\exists\).
\end{proof}
\begin{cor}
In particular, this shows that OWAFRS does not have the inclusion property (Equation \eqref{eq:propinc}), regardless of any non-trivial orness conditions.
\end{cor}
The following propositions show that the monotonicity and duality properties are still retained by the CFRS model, it thus retains all the properties of OWAFRS that are stated in \cite{d2015comprehensive}.
\begin{prop}[Relation monotonicity]
	\label{Prop:MonotoneRelationChoquet}
	Let \(R_1\subseteq R_2\in \mathcal{F}(X\cross X)\), \(A\in \mathcal{F}(X)\) and \(\mu\) a monotone measure on \(X\), then the following holds:
	\[\underline{\text{apr}}_{R_1,\mu}A \supseteq \underline{\text{apr}}_{R_2,\mu}A\text{ and }\overline{\text{apr}}_{R_1,\mu}A \subseteq \overline{\text{apr}}_{R_2,\mu}A.\]
\end{prop}
\begin{proof}
	Directly follows from the decreasingness of the implicator in the first argument, the increasingness of the t-norm and the monotonicity of the Choquet integral.
\end{proof}
\begin{prop}[Set monotonicity]
	Let \(R\in \mathcal{F}(X\cross X)\), \(A_1\subseteq A_2\in \mathcal{F}(X)\) and \(\mu\) a monotone measure on \(X\), then the following holds:
	\[\underline{\text{apr}}_{R,\mu}A_1 \subseteq \underline{\text{apr}}_{R,\mu}A_2\text{ and }\overline{\text{apr}}_{R,\mu}A_1 \subseteq \overline{\text{apr}}_{R,\mu}A_2.\]
\end{prop}
\begin{proof}
	Directly follows from the increasingness of the implicator in the second argument, the increasingness of the t-norm and the monotonicity of the Choquet integral.
\end{proof}

\begin{prop}[Duality]
	Let \(\imp\) be an implicator, \(\mathcal{C}\) the conjunctor induced by \(\imp\) and the standard negator $\mathcal{N}_s$, \(R\in\mathcal{F}(X\times X)\), \(\mu\) a monotone measure on \(X\) and \(A\in\mathcal{F}(X)\). Then the following holds:
	\[\underline{\text{apr}}_{R,\mu}A=\comp_{\mathcal{N}_s}\left(\overline{\text{apr}}_{R,\overline{\mu}} (\comp_{\mathcal{N}_s}(A))\right).\]
\end{prop}
\begin{proof}
	Using Proposition \ref{lemmaDuality}, the involutivity of \(\mathcal{N}_s\) and the definition of induced conjunctor we get:
	\begin{align*}
	\comp_{\mathcal{N}_s}\left(\overline{\text{apr}}_{R,\overline{\mu}} (\comp_{\mathcal{N}_s}(A))\right)(y)&=\mathcal{N}_s\left(\int
		\mathcal{C}(R(x,y),\mathcal{N}_s(A(x)))\diff\overline{\mu}(x)\right)\\
		&=\int\mathcal{N}_s\left[
		\mathcal{C}(R(x,y),\mathcal{N}_s(A(x)))\right]\diff\mu(x)\\
		&= \int\mathcal{N}_s\left[
		\mathcal{C}(R(x,y),\mathcal{N}_s(A(x)))\right]\diff\mu(x)\\
		&= \int\mathcal{N}_s\left[\mathcal{N}_s\left(
		\mathcal{I}(R(x,y),\mathcal{N}_s(\mathcal{N}_s(A(x))))\right)\right]\diff\mu(x)\\
		&=\int \mathcal{I}(R(x,y),A(x))\diff\mu(x)\\
		&=(\underline{\text{apr}}_{R,\mu}A)(y).
	\end{align*}
\end{proof}

\section{Application to classification}
\label{section:application}
\subsection{Classification using fuzzy rough sets}
The goal of classification is to predict the class of an instance, given a set of examples. More specifically, the set of examples is given in the form of a decision system \((X,\mathcal{A}\cup\{d\})\). For classification we assume that \(d\) is a categorical attribute; the attributes of \(\mathcal{A}\) can either be categorical of numerical. The problem of classification is then to predict for a new instance \(x\notin X\), for which the evaluations of the conditional attributes are given, the value of \(d(x)\), based on the decision system \((X,\mathcal{A}\cup\{d\})\). A simple algorithm for classification \cite{vluymans2019dealing}, using fuzzy rough sets, is to classify a test instance to the decision class for which it has the greatest membership to the lower approximation of that class. To calculate these lower approximations, we need a fuzzy relation \(R\in\mathcal{F}(X\times X)\) describing the similarity between instances based on the conditional attributes. When all the conditional attributes are numerical the following relation could be used:
\begin{equation}
	\label{indiscernibilityRelation}
R(x,y)=\frac{1}{|\mathcal{A}|}\sum_{a\in\mathcal{A}}R_a(x,y),
\end{equation}
where
\[R_a(x,y)=\max\left(1-\frac{\abs{a(y)-a(x)}}{\sigma_a},0\right),\]
and \(\sigma_a\) denotes the standard deviation of \(a\).\\
Let \(C\) be a decision class, since \(C\) is a crisp set, the membership degrees \(C(y)\) can only take values in \(\{0,1\}\). Using this fact, the antitonicity of an implicator in the first argument and \(\imp(1,1)=1\), we can rewrite the membership of an element \(x\) in the lower approximation as follows:
\begin{align*}
	(\underline{apr}_{R,\mu_l}C)(x)&= Agg\left(\underbrace{\imp(R(x,y),0)}_{y\notin A},\underbrace{\imp(R(x,y),1)}_{y\in A}\right)\\
	&=Agg\left(\underbrace{\imp(R(x,y),0)}_{y\notin A},\underbrace{1,\dots,1}_{y\in A}\right),
\end{align*}
where \(Agg\) is the aggregation operator corresponding to the monotone measure \(\mu_l\). Discarding the trailing one values from the previous equation, as was also done in \cite{vluymans2019weight}, gives us the following simpler definition of the lower approximation:
\[\underline{C}(x)=\underset{y\notin C}{Agg}\left(\imp(R(x,y),0)\right),\]
which in the case of the \L ukasiewicz, Kleene-Dienes and Reichenbach implicators reduces to
\begin{equation}
	\label{lowerApproximation}
	\underline{C}(x)=\underset{y\notin C}{Agg}(1-R(x,y)).
\end{equation}
\subsection{Experimental evaluation}
In this subsection we evaluate the measures introduced in Subsection \ref{section:measures}, when applied to classification.

\subsubsection{Setup} 
For ease of computation, we will make use of the lower approximation defined in Equation \eqref{lowerApproximation}, together with the indiscernibility relation from Equation \eqref{indiscernibilityRelation}. The aggregation operators we evaluate are:
\begin{itemize}
	\item fuzzy outlier removal, described by Equation \eqref{fuzzy removal}, using the minimum as the t-norm ($\mathcal{T}=\min$), denoted by FR,
	\item the Choquet integral with the WOWA measure from Equation \eqref{WOWA} using \(Q=Q_{add}\), denoted by WOWA,
	\item the Choquet integral with the measure from Equation \eqref{TSmeasure}, \(Q=Q_{add}\) (Example \ref{exmp: additive quantifier}) and \(t=0.3\), denoted by TS.
\end{itemize}
The baseline aggregation operators are minimum, average and OWA with \(Q=Q_{add}\) and each of these three are considered with outliers (denoted by Min, Avg and OWA) and without outliers (denoted by Mino, Avgo and OWAo), resulting in \(6\) different baseline operators. For \(O(x)\) we use the scores obtained from the normalised Local Outlier Factor (LOF)  algorithm \cite{normalizedLOF} which is implemented in the Pyod library \cite{zhao2019pyod}. The normalized LOF score of an instance \(x\) is calculated by fitting the LOF algorithm to all the members of the class that \(x\) is part of, and then using this model to calculate the LOF score of \(x\). FR and WOWA use these raw scores, while for Mino, Avgo and OWAo, the \(c*n\) samples with the highest outlier score are labelled as outliers, where \(n\) is the number of samples and the \emph{contamination} parameter \(c\) is set to \(0.1\). Except for these 9 algorithms, we will also test a combination of all of them. This combination algorithm, denoted by COMB, performs Leave-One-Out Cross-Validation (LOOCV) for each of these 9 algorithms on the training set and then chooses the best aggregation operator, evaluated using balanced accuracy
\[\frac{\text{true positive rate}+\text{true negative rate}}{2},\]
 to use on the test set. When multiple aggregation operators perform equally well, it chooses one of these randomly.  We evaluate the performance on 18 two-class datasets from the UCI-repository \cite{Dua:2019} by means of stratified 5-fold cross-validation. All of the datasets only have numerical features. The balanced accuracy
 is used as the performance measure.
\begin{table}[H]
	\begin{center}
		\begin{tabular}{l| c c l || l | c c l}
			Name & \# Feat. & \# Inst. & IR & Name & \# Feat. & \# Inst. & IR \\
			\hline
			appendicitis & 7 & 106 & 4.05 & pop-failures & 18 & 540 & 10.74 \\
			\hline
			banknote & 4 & 1372 & 1.25 & somerville & 6 & 143 & 1.17 \\
			\hline
			biodeg & 41 & 1055 & 1.96 & sonar & 60 & 208 & 1.14 \\
			\hline
			coimbra & 9 & 116 & 1.23 & spectf & 44 & 267 & 3.85 \\
			\hline
			debrecen & 19 & 1151 & 1.13 & sportsarticles & 59 & 1000 & 1.74 \\
			\hline
			divorce & 54 & 170 & 1.02 & transfusion & 4 & 748 & 3.20 \\
			\hline
			haberman & 3 & 306 & 2.78 & wdbc & 30 & 569 & 1.68 \\
			\hline
			ilpd & 10 & 579 & 2.51 & wisconsin & 9 & 683 & 1.86\\
			\hline
			ionosphere & 34 & 351 & 1.79 & wpbc & 32 & 138 & 3.93 \\
		\end{tabular}
	\end{center}
	\caption{Description of the selected datasets (\# Feat.\ = number of features, \# Inst.\ = number of instances, IR = imbalance ratio = number of instances in the majority class divided by the number of instances in the minority class).}
	\label{dataset_description}
\end{table}
\subsubsection{Results and discussion}
  The results of the cross-validation can be seen in Table \ref{resultsBenchmark} together with the mean and median over all datasets. A first impression that we get from these averages and medians is that the OWA-based operators (OWA, OWAo and WOWA) outperform all the other operators. Furthermore, the results indicate that the combination of all the algorithms (COMB) performs better than the individual ones. From this we can deduce that the best aggregation operator for the training set, also performs well on the test set. For the minimum-based (Min, Mino and FR) operators we get the impression that FR performs best, and that fully removing the ``outliers'' (Mino) is, on average, not a good decision. This last fact could be explained by the fact that some of the datasets may not really contain ``outliers'', and thus Mino is throwing away valuable information, which because of the sensitivity of the minimum to individual instances can alter the results considerably. \\
Table \ref{usedAlgosCOMB} displays which algorithms were used in the COMB algorithm for each dataset. The aggregation operators that were used most are Min, WOWA, and to a lesser extent FR, OWAo and TS. The rest (Mino, Avg, Avgo and OWA) were almost never used. \\
To discern if one of these methods really outperforms another method consistently and significantly we perform a two-sided Wilcoxon signed ranks test. The results of this are displayed in Figure \ref{heatmapBench}. First of all we observe that COMB is significantly different from Min, Mino, FR, Avg, Avgo and TS (\(p\leq 0.05\)). From the rank sums we can tell that COMB outperforms all of them. Secondly, we notice that Avg, Avgo and TS are all significantly (\(p<0.1\)) different from OWA, OWAo and WOWA. Looking at the rank sums tells us that Avg, Avgo and TS perform consistently worse than the OWA-based methods. Out of all the average-based methods and TS (Avg, Avgo and TS), the Wilcoxon test indicates that TS performs the best (\(p<0.02\)). For the minimum-based aggregation operators we can't really say anything about how they compare with the other operators. What the Wilcoxon test does show is that FR consistently outperforms Mino (\(p=0.042\)). From this we may conclude that when we choose to take into account ``outliers'', via an outlier detection algorithm, it may be preferable to ``fuzzy remove'' them (FR), instead of discarding them entirely (Mino). This also suggests that even when ``outliers'' are present in the dataset, FR still performs better. A possible explanation for the fact that FR outperforms Mino (on average) might be that when a dataset does not contain outliers, all the outlier scores are all close to \(0\), but there still need to be \(n*0.1\) outlier labels so Mino would throw away instances randomly, whereas FR would not, because it can ``see'' that all of the outlier scores are roughly zero.
\begin{landscape}
	\begin{table}[t!]
		\begin{center}
			\begin{tabular}{l|l|l|l|l|l|l|l|l|l|l}
				Dataset & Min & Mino & FR & Avg & Avgo & TS & OWA & OWAo & WOWA & COMB \\
				\hline
				\hline
				appendicitis & 0.739 & 0.704 & 0.773 & 0.768 & 0.768 & 0.793 & 0.768 & 0.774 & 0.774 & 0.754 \\
				\hline
				banknote & 0.999 & 0.991 & 0.990 & 0.872 & 0.872 & 0.908 & 0.916 & 0.912 & 0.906  & 0.999\\
				\hline
				biodeg & 0.817 & 0.815 & 0.820 & 0.690 & 0.707 & 0.739 & 0.753 & 0.756 & 0.763 & 0.820  \\
				\hline
				coimbra & 0.633 & 0.634 & 0.633 & 0.596 & 0.595 & 0.610 & 0.633 & 0.636 & 0.620 & 0.615 \\
				\hline
				debrecen & 0.616 & 0.619 & 0.626 & 0.576 & 0.580 & 0.616 & 0.619 & 0.622 & 0.629 & 0.630\\
				\hline
				divorce & 0.976 & 0.976 & 0.976 & 0.976 & 0.976 & 0.976 & 0.976 & 0.976 & 0.976 & 0.976 \\
				\hline
				haberman & 0.540 & 0.581 & 0.583 & 0.625 & 0.628 & 0.602 & 0.622 & 0.634 & 0.629 & 0.634\\
				\hline
				ilpd & 0.620 & 0.613 & 0.609 & 0.607 & 0.615 & 0.624 & 0.639 & 0.649 & 0.656 & 0.644 \\
				\hline
				ionosphere & 0.902 & 0.900 & 0.897 & 0.837 & 0.854 & 0.863 & 0.859 & 0.867 & 0.850 & 0.902 \\
				\hline
				pop-failures & 0.616 & 0.628 & 0.639 & 0.797 & 0.778 & 0.802 & 0.788 & 0.783 & 0.801 & 0.802 \\
				\hline
				somerville & 0.566 & 0.567 & 0.573 & 0.623 & 0.640 & 0.612 & 0.624 & 0.655 & 0.613 & 0.583\\
				\hline
				sonar & 0.840 & 0.832 & 0.824 & 0.751 & 0.747 & 0.772 & 0.806 & 0.792 & 0.789 & 0.780  \\
				\hline
				spectf & 0.619 & 0.617 & 0.621 & 0.580 & 0.590 & 0.606 & 0.620 & 0.632 & 0.663 & 0.637  \\
				\hline
				sportsarticles & 0.729 & 0.744 & 0.745 & 0.775 & 0.775 & 0.779 & 0.782 & 0.789 & 0.783& 0.789  \\
				\hline
				transfusion & 0.594 & 0.562 & 0.571 & 0.659 & 0.649 & 0.663 & 0.672 & 0.640 & 0.641& 0.652 \\
				\hline
				wdbc & 0.942 & 0.943 & 0.943 & 0.911 & 0.910 & 0.915 & 0.923 & 0.927 & 0.930& 0.943  \\
				\hline
				wisconsin & 0.963 & 0.958 & 0.958 & 0.894 & 0.892 & 0.922 & 0.926 & 0.926 & 0.926 & 0.967 \\
				\hline
				wpbc & 0.511 & 0.519 & 0.519 & 0.630 & 0.628 & 0.609 & 0.616 & 0.625 & 0.620 & 0.605 \\
				\hline
				\hline
				Average & 0.735 & 0.733 & 0.739 & 0.731 & 0.734 & 0.745 & 0.752 & 0.755 & 0.754 & 0.763\\
				\hline
				Median & 0.681 & 0.669 & 0.692 & 0.721 & 0.727 & 0.756 & 0.760 & 0.765 & 0.768 & 0.767 \\
			\end{tabular}
		\end{center}
		\caption{Mean balanced accuracy results from performing 5-fold cross-validation.}
		\label{resultsBenchmark}
	\end{table}
\end{landscape}
\begin{table}[htp]
	\begin{center}
		\begin{tabular}{l||c|c|c|c|c|c|c|c|c}
			Dataset & Min & Mino & FR & Avg & Avgo & TS & OWA & OWAo & WOWA \\
			\hline
			appendicitis & 0 & 0 & 1 & 1 & 0 & 0 & 0 & 1 & 2 \\
			\hline
			banknote & 5 & 0 & 0 & 0 & 0 & 0 & 0 & 0 & 0 \\
			\hline
			biodeg & 0 & 0 & 5 & 0 & 0 & 0 & 0 & 0 & 0 \\
			\hline
			coimbra & 0 & 0 & 0 & 0 & 0 & 3 & 0 & 0 & 2 \\
			\hline
			debrecen & 0 & 0 & 1 & 0 & 0 & 0 & 0 & 0 & 4 \\
			\hline
			divorce & 5 & 0 & 0 & 0 & 0 & 0 & 0 & 0 & 0 \\
			\hline
			haberman & 0 & 0 & 0 & 0 & 1 & 0 & 1 & 1 & 2 \\
			\hline
			ilpd & 0 & 0 & 0 & 0 & 0 & 0 & 1 & 2 & 2 \\
			\hline
			ionosphere & 4 & 0 & 1 & 0 & 0 & 0 & 0 & 0 & 0 \\
			\hline
			pop-failures & 0 & 0 & 0 & 0 & 0 & 0 & 0 & 2 & 3 \\
			\hline
			somerville & 0 & 0 & 1 & 1 & 1 & 0 & 0 & 1 & 1 \\
			\hline
			sonar & 2 & 0 & 0 & 0 & 0 & 0 & 1 & 1 & 1 \\
			\hline
			spectf & 1 & 0 & 0 & 0 & 0 & 0 & 0 & 0 & 4 \\
			\hline
			sportsarticles & 0 & 0 & 0 & 0 & 0 & 0 & 0 & 5 & 0 \\
			\hline
			transfusion & 0 & 0 & 0 & 1 & 0 & 3 & 0 & 0 & 1 \\
			\hline
			wdbc & 1 & 1 & 3 & 0 & 0 & 0 & 0 & 0 & 0 \\
			\hline
			wisconsin & 4 & 0 & 1 & 0 & 0 & 0 & 0 & 0 & 0 \\
			\hline
			wpbc & 0 & 0 & 0 & 1 & 1 & 2 & 1 & 0 & 0 \\
			\hline
			\hline
			sum & 22 & 1 & 13 & 4 &  3 & 8 & 4 & 11 & 22\\
		\end{tabular}
	\end{center}
	\caption{Number of folds for which each aggregation function is used in the COMB algorithm.}
	\label{usedAlgosCOMB}
\end{table}
\section{Conclusion and future work}
\label{section: conclusion and future work}
From a theoretical perspective, Choquet-based fuzzy rough sets are a natural extension of OWA-based fuzzy rough sets as they preserve all of their known theoretical properties. \\
From a classification perspective, Choquet-based fuzzy rough sets yield more flexibility than OWA-based fuzzy rough sets. We have seen that they allow to combine fuzzy rough set models and outlier detection algorithms more smoothly, through fuzzy removal or weighted ordered weighted averaging. They do so by directly using the outlier scores instead of having to discretize them into outlier labels. The benefit of this is that we don't have to optimize the \emph{contamination} parameter that is used to create the outlier labels, and that this might yield a higher classification accuracy. Moreover, we have seen that choosing a good measure for a particular dataset can be done by selecting the one that performs the best on the training set.\\ 
The algorithms based on the measures introduced in this paper can still be improved, for example, by tuning the weight vector for the WOWA operator, as well by experimenting with different types of outlier detection algorithms and transformations of the outlier scores.\\
The measures proposed in this work are described by at most \(2n\) parameters, where \(n\) is the length of the vector that needs to be aggregated, but an arbitrary monotone measure depends on \(2^n-2\) parameters. Therefore, we still haven't used the full generality of the Choquet integral, so one could search for measures that depend on higher dimensional characteristics of the training set, instead of the one-dimensional outlier scores. 
\section*{Acknowledgment}
The research reported in this paper was conducted with the financial support
of the Odysseus programme of the Research Foundation – Flanders (FWO). 
	\begin{figure}[htp]
	\centering
	\includegraphics[width=14cm]{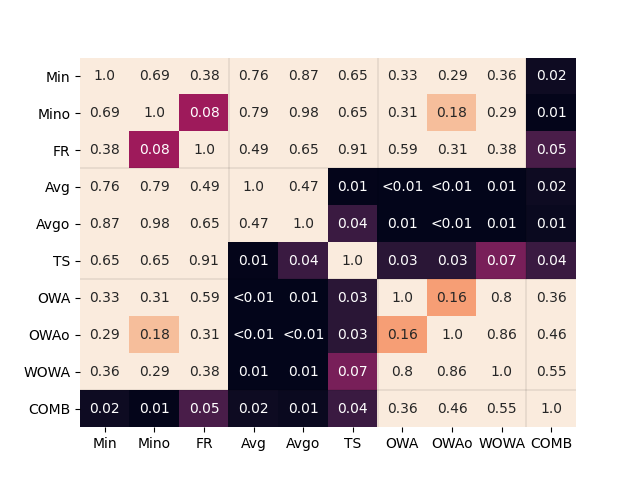}
	\caption{Heatmap of the \(p\)-values from the pairwise two-sided Wilcoxon signed rank test.}
	\label{heatmapBench}
\end{figure}
\newpage
\bibliography{mybibfile}

\begin{thebibliography}{10}
\expandafter\ifx\csname url\endcsname\relax
  \def\url#1{\texttt{#1}}\fi
\expandafter\ifx\csname urlprefix\endcsname\relax\def\urlprefix{URL }\fi
\expandafter\ifx\csname href\endcsname\relax
  \def\href#1#2{#2} \def\path#1{#1}\fi

\bibitem{pawlak1982rough}
Z.~Pawlak, Rough sets, International journal of computer \& information
  sciences 11~(5) (1982) 341--356.

\bibitem{dubois1990rough}
D.~Dubois, H.~Prade, Rough fuzzy sets and fuzzy rough sets, International
  Journal of General System 17~(2-3) (1990) 191--209.

\bibitem{vluymans2015applications}
S.~Vluymans, L.~D’eer, Y.~Saeys, C.~Cornelis, Applications of fuzzy rough set
  theory in machine learning: a survey, Fundamenta Informaticae 142~(1-4)
  (2015) 53--86.

\bibitem{cornelis2007vaguely}
C.~Cornelis, M.~De~Cock, A.~M. Radzikowska, Vaguely quantified rough sets, in:
  International Workshop on Rough Sets, Fuzzy Sets, Data Mining, and
  Granular-Soft Computing, Springer, 2007, pp. 87--94.

\bibitem{FERNANDEZSALIDO2003547}
J.~{Fernández Salido}, S.~Murakami,
  \href{https://www.sciencedirect.com/science/article/pii/S0165011403000034}{On
  \(\beta\)-precision aggregation}, Fuzzy Sets and Systems 139~(3) (2003)
  547--558.
\newblock \href {https://doi.org/https://doi.org/10.1016/S0165-0114(03)00003-4}
  {\path{doi:https://doi.org/10.1016/S0165-0114(03)00003-4}}.
\newline\urlprefix\url{https://www.sciencedirect.com/science/article/pii/S0165011403000034}

\bibitem{FERNANDEZSALIDO2}
J.~{Fernández Salido}, S.~Murakami,
  \href{https://www.sciencedirect.com/science/article/pii/S0165011403001246}{Rough
  set analysis of a general type of fuzzy data using transitive aggregations of
  fuzzy similarity relations}, Fuzzy Sets and Systems 139~(3) (2003) 635--660.
\newblock \href {https://doi.org/https://doi.org/10.1016/S0165-0114(03)00124-6}
  {\path{doi:https://doi.org/10.1016/S0165-0114(03)00124-6}}.
\newline\urlprefix\url{https://www.sciencedirect.com/science/article/pii/S0165011403001246}

\bibitem{mieszkowicz2004variable}
A.~Mieszkowicz-Rolka, L.~Rolka, Variable precision fuzzy rough sets, in:
  Transactions on Rough Sets I, Springer, 2004, pp. 144--160.

\bibitem{YAO201458}
Y.~Yao, J.~Mi, Z.~Li,
  \href{https://www.sciencedirect.com/science/article/pii/S0165011413002753}{A
  novel variable precision \((\theta,\sigma)\)-fuzzy rough set model based on
  fuzzy granules}, Fuzzy Sets and Systems 236 (2014) 58--72, theme: Algebraic
  Aspects of Fuzzy Sets.
\newblock \href {https://doi.org/https://doi.org/10.1016/j.fss.2013.06.012}
  {\path{doi:https://doi.org/10.1016/j.fss.2013.06.012}}.
\newline\urlprefix\url{https://www.sciencedirect.com/science/article/pii/S0165011413002753}

\bibitem{HU20104384}
Q.~Hu, S.~An, D.~Yu,
  \href{https://www.sciencedirect.com/science/article/pii/S0020025510003282}{Soft
  fuzzy rough sets for robust feature evaluation and selection}, Information
  Sciences 180~(22) (2010) 4384--4400.
\newblock \href {https://doi.org/https://doi.org/10.1016/j.ins.2010.07.010}
  {\path{doi:https://doi.org/10.1016/j.ins.2010.07.010}}.
\newline\urlprefix\url{https://www.sciencedirect.com/science/article/pii/S0020025510003282}

\bibitem{HADRANI202037}
A.~Hadrani, K.~Guennoun, R.~Saadane, M.~Wahbi,
  \href{https://www.sciencedirect.com/science/article/pii/S1389041720300255}{Fuzzy
  rough sets: Survey and proposal of an enhanced knowledge representation model
  based on automatic noisy sample detection}, Cognitive Systems Research 64
  (2020) 37--56.
\newblock \href {https://doi.org/https://doi.org/10.1016/j.cogsys.2020.05.001}
  {\path{doi:https://doi.org/10.1016/j.cogsys.2020.05.001}}.
\newline\urlprefix\url{https://www.sciencedirect.com/science/article/pii/S1389041720300255}

\bibitem{data-distribution-awareFR}
S.~An, Q.~Hu, W.~Pedrycz, P.~Zhu, E.~C.~C. Tsang, Data-distribution-aware fuzzy
  rough set model and its application to robust classification, IEEE
  Transactions on Cybernetics 46~(12) (2016) 3073--3085.
\newblock \href {https://doi.org/10.1109/TCYB.2015.2496425}
  {\path{doi:10.1109/TCYB.2015.2496425}}.

\bibitem{PGDFRS}
S.~An, Q.~Hu, C.~Wang,
  \href{https://www.sciencedirect.com/science/article/pii/S1568494620310024}{Probability
  granular distance-based fuzzy rough set model}, Applied Soft Computing 102
  (2021) 107064.
\newblock \href {https://doi.org/https://doi.org/10.1016/j.asoc.2020.107064}
  {\path{doi:https://doi.org/10.1016/j.asoc.2020.107064}}.
\newline\urlprefix\url{https://www.sciencedirect.com/science/article/pii/S1568494620310024}

\bibitem{cornelis2010ordered}
C.~Cornelis, N.~Verbiest, R.~Jensen, Ordered weighted average based fuzzy rough
  sets, in: International Conference on Rough Sets and Knowledge Technology,
  Springer, 2010, pp. 78--85.

\bibitem{grabisch2010decade}
M.~Grabisch, C.~Labreuche, A decade of application of the {Choquet} and
  {Sugeno} integrals in multi-criteria decision aid, Annals of Operations
  Research 175~(1) (2010) 247--286.

\bibitem{normalizedLOF}
H.-P. Kriegel, P.~Kroger, E.~Schubert, A.~Zimek, Interpreting and unifying
  outlier scores, in: Proceedings of the 2011 SIAM International Conference on
  Data Mining, SIAM, 2011, pp. 13--24.

\bibitem{fuzzysetsss}
L.~A. Zadeh, Fuzzy sets, Information and Control (1965).

\bibitem{radzikowska2002comparative}
A.~M. Radzikowska, E.~E. Kerre, A comparative study of fuzzy rough sets, Fuzzy
  sets and systems 126~(2) (2002) 137--155.

\bibitem{yagerOWA}
R.~R. Yager, On ordered weighted averaging aggregation operators in
  multicriteria decisionmaking, IEEE Transactions on systems, Man, and
  Cybernetics 18~(1) (1988) 183--190.

\bibitem{beliakov2007aggregation}
G.~Beliakov, A.~Pradera, T.~Calvo, et~al., Aggregation functions: A guide for
  practitioners, Vol. 221, Springer, 2007.

\bibitem{wang2010generalized}
Z.~Wang, G.~J. Klir, Generalized measure theory, Vol.~25, Springer Science \&
  Business Media, 2010.

\bibitem{zadehFuzzyQuantifier}
L.~A. Zadeh, A computational approach to fuzzy quantifiers in natural
  languages, in: Computational linguistics, Elsevier, 1983, pp. 149--184.

\bibitem{yager1996quantifier}
R.~R. Yager, Quantifier guided aggregation using owa operators, International
  Journal of Intelligent Systems 11~(1) (1996) 49--73.

\bibitem{vluymans2019dealing}
S.~Vluymans, Dealing with imbalanced and weakly labelled data in machine
  learning using fuzzy and rough set methods, Springer, 2019.

\bibitem{WOWA}
V.~Torra, The weighted owa operator, International Journal of Intelligent
  Systems 12~(2) (1997) 153--166.

\bibitem{WOWA<Choquet}
V.~Torra, On some relationships between the wowa operator and the {Choquet}
  integral, in: Proceedings of the IPMU 1998 Conference, Paris, France,
  Citeseer, 1998, pp. 818--824.

\bibitem{p-symmetry}
P.~Miranda, M.~Grabisch, P.~Gil, p-symmetric fuzzy measures, International
  Journal of Uncertainty, Fuzziness and Knowledge-Based Systems 10~(supp01)
  (2002) 105--123.

\bibitem{d2015comprehensive}
L.~D'eer, N.~Verbiest, C.~Cornelis, L.~Godo, A comprehensive study of
  implicator--conjunctor-based and noise-tolerant fuzzy rough sets:
  definitions, properties and robustness analysis, Fuzzy Sets and Systems 275
  (2015) 1--38.

\bibitem{vluymans2019weight}
S.~Vluymans, N.~Mac~Parthal{\'a}in, C.~Cornelis, Y.~Saeys, Weight selection
  strategies for ordered weighted average based fuzzy rough sets, Information
  Sciences 501 (2019) 155--171.

\bibitem{zhao2019pyod}
Y.~Zhao, Z.~Nasrullah, Z.~Li,
  \href{http://jmlr.org/papers/v20/19-011.html}{Pyod: A {Python} toolbox for
  scalable outlier detection}, Journal of Machine Learning Research 20~(96)
  (2019) 1--7.
\newline\urlprefix\url{http://jmlr.org/papers/v20/19-011.html}

\bibitem{Dua:2019}
D.~Dua, C.~Graff, \href{http://archive.ics.uci.edu/ml}{{UCI} machine learning
  repository} (2017).
\newline\urlprefix\url{http://archive.ics.uci.edu/ml}

\end{thebibliography}

\end{document}